\documentclass[letterpaper]{article}
\usepackage{aaai18}
\usepackage{times}
\usepackage{helvet}
\usepackage{courier}
\usepackage{url}
\usepackage{graphicx}
\usepackage{amsthm}
\usepackage{thmtools}
\usepackage{thm-restate}
\usepackage{wrapfig}
\usepackage[title]{appendix}
\usepackage{tabularx}
\declaretheorem[name=Definition]{definition}
\declaretheorem[name=Lemma]{lemma}
\declaretheorem[name=Remark]{remark}

\usepackage[margin=1in]{geometry}
\usepackage{xcolor}
\usepackage{todonotes}
\usepackage[capitalise]{cleveref}
\usepackage{algorithm}
\usepackage[noend]{algpseudocode}
\usepackage{amsmath}
\usepackage{amssymb}
\usepackage{siunitx}
\usepackage{booktabs}
\usepackage{caption}
\usepackage{subcaption}
\usepackage{graphicx}
\usepackage{array}
\usepackage{footnote}
\usepackage{bigfoot}
\makesavenoteenv{tabular}

\newcommand{\argmax}{\operatornamewithlimits{arg\,max}}
\newcommand{\argmin}{\operatornamewithlimits{arg\,min}}
\newcommand{\reals}{\ensuremath{\mathbb{R}}}
\newcommand{\obsspace}{\ensuremath{\mathcal{O}}}
\newcommand{\obsdist}{\ensuremath{\mathcal{Z}}}

\newcommand{\citet}[1]
{\citeauthor{#1} ̃\shortcite{#1}}
\newcommand{\citep}{\cite}

\crefname{appsec}{Appendix}{Appendices}

\newcommand{\result}[4]{
    $#1\pm#2$ & {\color{green!#3!red}\rule{#4cm}{8pt}}
}
\newcommand{\noresult}{ & }

\title{Online algorithms for POMDPs with continuous state, action, and observation spaces}

\author{ {\bf Zachary N. Sunberg\thanks{\hspace{1ex}\{\texttt{zsunberg}, \texttt{mykel}\}\texttt{@stanford.edu}}\and {\bf Mykel J. Kochenderfer\footnotemark[1]}} \\
Aeronautics and Astronautics Dept.\\
Stanford University \\
Stanford, CA 94305 \\
}

\begin{document}

\nocopyright
\maketitle

\begin{abstract}
    Online solvers for partially observable Markov decision processes have been applied to problems with large discrete state spaces, but continuous state, action, and observation spaces remain a challenge. This paper begins by investigating double progressive widening (DPW) as a solution to this challenge. However, we prove that this modification alone is not sufficient because the belief representations in the search tree collapse to a single particle causing the algorithm to converge to a policy that is suboptimal regardless of the computation time. This paper proposes and evaluates two new algorithms, POMCPOW and PFT-DPW, that overcome this deficiency by using weighted particle filtering. Simulation results show that these modifications allow the algorithms to be successful where previous approaches fail.
\end{abstract}

\section{Introduction}

The partially observable Markov decision process (POMDP) is a flexible mathematical framework for representing sequential decision problems~\cite{littman1995learning,thrun2005probabilistic}.
Once a problem has been formalized as a POMDP, a wide range of solution techniques can be used to solve it.
In a POMDP, at each step in time, an agent selects an action causing the state to change stochastically to a new value based only on the current state and action.
The agent seeks to maximize the expectation of the reward, which is a function of the state and action.
However, the agent cannot directly observe the state, and makes decisions based only on observations that are stochastically generated by the state.

Many \emph{offline} methods have been developed to solve small and moderately sized POMDPs \cite{kurniawati2008sarsop}. Solving larger POMDPs generally requires the use of \emph{online} methods \cite{silver2010pomcp,somani2013despot,kurniawati2016online}.
One widely used online algorithm is partially observable Monte Carlo planning (POMCP) \cite{silver2010pomcp}, which is an extension to Monte Carlo tree search that implicitly uses an unweighted particle filter to represent beliefs in the search tree.

POMCP and other online methods can accomodate continuous state spaces, and there has been recent work on solving problems with continuous action spaces \cite{seiler2015online}.
However, there has been less progress on problems with continuous observation spaces.
This paper presents two similar algorithms which address the challenge of solving POMDPs with continuous state, action, and observation spaces.
The first is based on POMCP and is called partially observable Monte Carlo planning with observation widening (POMCPOW). 
The second solves the belief-space MDP and is called particle filter trees with double progressive widening (PFT-DPW).

\begin{figure}[t]
\begin{center}
    \includegraphics[width=0.9\columnwidth]{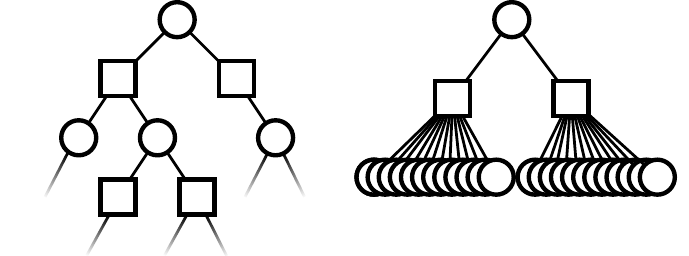}
\end{center}
\caption{POMCP tree for a discrete POMDP (left), and for a POMDP with a continuous observation space (right). Because the observation space is continuous, each simulation creates a new observation node and the tree cannot extend deeper.}
\label{fig:ctree}
\end{figure}

There are two challenges that make tree search difficult in continuous spaces.
The first is that, since the probability of sampling the same real number twice from a continuous random variable is zero, the width of the planning trees explodes on the first step, causing them to be too shallow to be useful (see \cref{fig:ctree}).
POMCPOW and PFT-DPW resolve this issue with a technique called double progressive widening (DPW) \cite{couetoux2011double}. 
The second issue is that, even when DPW is applied, the belief representations used by current solvers collapse to a single state particle, resulting in overconfidence.
As a consequence, the solutions obtained resemble QMDP policies, and there is no incentive for information gathering.
POMCPOW and PFT-DPW overcome this issue by using the observation model to weight the particles used to represent beliefs.

This paper proceeds as follows: \Cref{sec:prior} provides an overview of previous online POMDP approaches. \Cref{sec:background} provides a brief introduction to POMDPs and Monte Carlo tree search. \Cref{sec:algorithms} presents several algorithms for solving POMDPs on continuous spaces, discusses theoretical and practical aspects of their behavior. \Cref{sec:experiments} then gives experimental validation of the algorithms.

\section{Prior Work} \label{sec:prior}

Considerable progress has been made in solving large POMDPs.
Initially, exact offline solutions to problems with only a few discrete states, actions, and observations were sought by using value iteration and taking advantage of the convexity of the value function \cite{kaelbling1998planning}, although solutions to larger problems were also explored using Monte Carlo simulation and interpolation between belief states \cite{thrun1999monte}.
Many effective offline planners for discrete problems use point based value iteration, where a selection of points in the belief space are used for value function approximation,  \cite{kurniawati2008sarsop}.
Offline solutions for problems with continuous state and observation spaces have also been proposed \cite{bai2014integrated,brechtel2013solving}.

There are also various solution approaches that are applicable to specific classes of POMDPs, including continuous problems.
For example, \citet{platt2010belief} simplify planning in large domains by assuming that the most likely observation will always be received, which can provide an acceptable approximation in some problems with unimodal observation distributions.
\citet{morere2016bayesian} solve a monitoring problem with continuous spaces with a Gaussian process belief update.
\citet{hoey2005solving} propose a method for partitioning large observation spaces without information loss, but demonstrate the method only on small state and action spaces that have a modest number of conditional plans.
Other methods involve motion-planning techniques \cite{melchior2007particle,prentice2009belief,bry2011rapidly}.
In particular, \citet{agha2011firm} present a method to take advantage of the existence of a stabilizing controller in belief space planning.
\citet{van2012motion} perform local optimization with respect to uncertainty on a pre-computed path, and \citet{indelman2015planning} devise a hierarchical approach that handles uncertainty in both the robot's state and the surrounding environment.

General purpose online algorithms for POMDPs have also been proposed.
Many early online algorithms focused on point-based belief tree search with heuristics for expanding the trees \cite{ross2008online}.
The introduction of POMCP \cite{silver2010pomcp} caused a pivot toward the simple and fast technique of using the same simulations for decision-making and using beliefs implicitly represented as unweighted collections of particles.
Determinized sparse partially observable tree (DESPOT) is a similar approach that attempts to achieve better performance by analyzing only a small number of random outcomes in the tree \cite{somani2013despot}.
Adaptive belief tree (ABT) was designed specifically to accommodate changes in the environment without having to replan from scratch \cite{kurniawati2016online}.

These methods can all easily handle continuous state spaces \cite{goldhoorn2014continuous}, but they must be modified to extend to domains with continuous action or observation spaces.
Though DESPOT has demonstrated effectiveness on some large problems, since it uses unweighted particle beliefs in its search tree, it struggles with continuous information gathering problems as will be shown in \cref{sec:experiments}.
ABT has been extended to use generalized pattern search for selecting locally optimal continuous actions, an approach which is especially effective in problems where high precision is important \cite{seiler2015online}, but also uses unweighted particle beliefs.
Continuous observation Monte Carlo tree search (COMCTS) constructs observation classification trees to automatically partition the observation space in a POMCP-like approach, however it did not perform much better than a Monte Carlo rollout approach in experiments \cite{pas2012simulation}.

Although research has yielded effective solution techniques for many classes of problems, there remains a need for simple, general purpose online POMDP solvers that can handle continuous spaces, especially continuous observation spaces.


\section{Background} \label{sec:background}

This section reviews mathematical formulations for sequential decision problems and some existing solution approaches. The discussion assumes familiarity with Markov decision processes~\cite{kochenderfer2015decision}, particle filtering~\cite{thrun2005probabilistic}, and Monte Carlo tree search~\cite{browne2012survey}, but reviews some details for clarity.

\subsection{POMDPs}

The Markov decision process (MDP) and partially observable Markov decision process (POMDP)  can represent a wide range of sequential decision making problems. In a Markov decision process, an agent takes actions that affect the state of the system and seeks to maximize the expected value of the rewards it collects \cite{kochenderfer2015decision}. Formally, an MDP is defined by the 5-tuple $(\mathcal{S}, \mathcal{A}, \mathcal{T}, \mathcal{R}, \gamma)$, where $\mathcal{S}$ is the state space, $\mathcal{A}$ is the action space, $\mathcal{T}$ is the transition model, $\mathcal{R}$ is the reward function, and $\gamma$ is the discount factor.
The transition model can be encoded as a set of probabilities, specifically $\mathcal{T}(s' \mid s, a)$ denotes the probability that the system will transition to state $s'$ given that action $a$ is taken in state $s$. In continuous problems, $\mathcal{T}$ is defined by probability density functions.

In a POMDP, the agent cannot directly observe the state.
Instead, the agent only has access to observations that are generated probabilistically based on the actions and latent true states.
A POMDP is defined by the 7-tuple $(\mathcal{S}, \mathcal{A}, \mathcal{T}, \mathcal{R}, \obsspace, \obsdist, \gamma)$, where $\mathcal{S}$, $\mathcal{A}$, $\mathcal{T}$, $\mathcal{R}$, and $\gamma$ have the same meaning as in an MDP.
Additionally, $\obsspace$, is the observation space, and $\obsdist$ is the observation model.
$\obsdist(o \mid s, a, s')$ is the probability or probability density of receiving observation $o$ in state $s'$ given that the previous state and action were $s$ and $a$.

Information about the state may be inferred from the entire history of previous actions and observations and the initial information, $b_0$.
Thus, in a POMDP, the agent's policy is a function mapping each possible history, $h_t = (b_0, a_0, o_1, a_1, o_2, \dots, a_{t-1}, o_t)$ to an action.
In some cases, each state's probability can be calculated based on the history.
This distribution is known as a \emph{belief}, with $b_t(s)$ denoting the probability of state $s$.

The belief is a sufficient statistic for optimal decision making.
That is, there exists a policy, $\pi^*$ such that, when $a_t = \pi^*(b_t)$, the expected cumulative reward or ``value function'' is maximized for the POMDP.
Given the POMDP model, each subsequent belief can be calculated using Bayes' rule \cite{kaelbling1998planning,kochenderfer2015decision}.
However, the exact update is computationally intensive, so approximate approaches such as particle filtering are usually used in practice~\cite{thrun2005probabilistic}.

\subsubsection{Generative Models}

For many problems, it can be difficult to explicitly determine or represent the probability distributions $\mathcal{T}$ or $\obsdist$.
Some solution approaches, however, only require samples from the state transitions and observations.
A generative model, $G$, stochastically generates a new state, reward, and observation in the partially observable case, given the current state and action, that is $s', r = G(s,a)$ for an MDP, or $s', o, r = G(s, a)$ for a POMDP.
A generative model implicitly defines $\mathcal{T}$ and $\obsdist$, even when they cannot be explicitly represented.

\subsubsection{Belief MDPs}

Every POMDP is equivalent to an MDP where the state space of the MDP is the space of possible beliefs.
The reward function of this "belief MDP" is the expectation of the state-action reward function with respect to the belief.
The Bayesian update of the belief serves as a generative model for the belief space MDP.

\subsection{MCTS with Double Progressive Widening}

Monte Carlo Tree Search (MCTS) is an effective and widely studied algorithm for online decision-making \cite{browne2012survey}.
It works by incrementally creating a policy tree consisting of alternating layers of state nodes and action nodes using a generative model $G$ and estimating the state-action value function, $Q(s,a)$, at each of the action nodes. The Upper Confidence Tree (UCT) version expands the tree by selecting nodes that maximize the upper confidence bound
\begin{equation} \label{eqn:ucb}
    UCB(s,a) = Q(s,a) + c \sqrt{\frac{\log N(s)}{N(s,a)}}
\end{equation}
where $N(s,a)$ is the number of times the action node has been visited, $N(s) = \sum_{a \in \mathcal{A}} N(s,a)$, and $c$ is a problem-specific parameter that governs the amount of exploration in the tree \cite{browne2012survey}.

\subsubsection{Double Progressive Widening}

In cases where the action and state spaces are large or continuous, the MCTS algorithm will produce trees that are very shallow.
In fact, if the action space is continuous, the UCT algorithm will never try the same action twice (observe that $\lim_{N(s,a) \to 0} UCB(s,a) = \infty$, so untried actions are always favored).
Moreover, if the state space is continuous and the transition probability density is finite, the probability of sampling the same state twice from $G$ is zero.
Because of this, simulations will never pass through the same state node twice and a tree below the first layer of state nodes will never be constructed.

In progressive widening, the number of children of a node is artificially limited to $k N^\alpha$ where $N$ is the number of times the node has been visited and $k$ and $\alpha$ are hyper-parameters (see \cref{sec:hyper})~\cite{couetoux2011double}.
Originally, progressive widening was applied to the action space and was found to be especially effective when a set of preferred actions was tried first \cite{browne2012survey}.
The term \emph{double} progressive widening refers to progressive widening in both the state and action space.
When the number of state nodes is greater than $k N^\alpha$, instead of simulating a new state transition, one of the previously generated states is chosen with probability proportional to the number of times it has been previously generated.

\subsection{POMCP}

A conceptually straightforward way to solve a POMDP using MCTS is to apply it to the corresponding belief MDP.
Indeed, many tree search techniques have been applied to POMDP problems in this way \cite{ross2008online}.
However, when the Bayesian belief update is used, this approach is computationally expensive.
POMCP and its successors, DESPOT and ABT, can tackle problems many times larger than their predecessors because they use state trajectory simulations, rather than full belief trajectories, to build the tree.

Each of the nodes in a POMCP tree corresponds to a history proceeding from the root belief and terminating with an action or observation.
In the search phase of POMCP tree construction, state trajectories are simulated through this tree.
At each action node, the rewards from the simulations that pass through the node are used to estimate the $Q$ function.
This simple approach has been shown to work well for large discrete problems \cite{silver2010pomcp}.
However, when the action or observation space is continuous, the tree degenerates and does not extend beyond a single layer of nodes because each new simulation produces a new branch.

\section{Algorithms} \label{sec:algorithms}

This section presents several MCTS algorithms for POMDPs including the new POMCPOW and PFT-DPW approaches.

The three algorithms in this section share a common structure.
For all algorithms, the entry point for the decision making process is the \textproc{Plan} procedure, which takes the current belief, $b$, as an input (\textproc{Plan} differs slightly for PFT-DPW in \cref{alg:pft}).
The algorithms also share the same \textproc{ActionProgWiden} function to control progressive widening of the action space.
These components are listed in Listing \ref{alg:common}.
The difference between the algorithms is in the \textproc{Simulate} function.

The following variables are used in the listings and text:
$h$ represents a history $(b, a_1, o_1, \dots a_k, o_k)$, and $ha$ and $hao$ are shorthand for histories with $a$ and $(a,o)$ appended to the end, respectively;
$d$ is the depth to explore, with $d_\text{max}$ the maximum depth;
$C$ is a list of the children of a node (along with the reward in the case of PFT-DPW);
$N$ is a count of the number of visits; and $M$ is a count of the number of times that a history has been generated by the model.
The list of states associated with a node is denoted $B$, and $W$ is a list of weights corresponding to those states.
Finally, $Q(ha)$ is an estimate of the value of taking action $a$ after observing history $h$.
$C$, $N$, $M$, $B$, $W$, and $Q$ are all implicitly initialized to \num{0} or $\emptyset$.
The \textproc{Rollout} procedure, runs a simulation with a default rollout policy, which can be based on the history or fully observed state for $d$ steps and returns the discounted reward.

\begin{algorithm}[t]
    \floatname{algorithm}{Listing}
    \caption{Common procedures} \label{alg:common}
    \begin{algorithmic}[1]
        \Procedure{Plan}{$b$}
            \For{$i \in 1:n$}
                \State $s \gets \text{sample from }b$
                \State $\Call{Simulate}{s, b, d_\text{max}}$
            \EndFor
            \State $\textbf{return } \underset{a}{\argmax}\, Q(ba)$
        \EndProcedure

        \Procedure {ActionProgWiden}{$h$}
            \If{$|C(h)| \leq k_a N(h)^{\alpha_a}$}
                \State $a \gets \Call{NextAction}{h}$
                \State $C(h) \gets C(h) \cup \{a\}$
            \EndIf
            \State $\textbf{return } \underset{a \in C(h)}{\argmax}\, Q(ha) + c \sqrt{\frac{\log N(h)}{N(ha)}}$
        \EndProcedure

    \end{algorithmic}
\end{algorithm}

\subsection{POMCP-DPW}

The first algorithm that we consider is POMCP with double progressive widening (POMCP-DPW).
In this algorithm, listed in \cref{alg:pomcpdpw}, the number of new children sampled from any node in the tree is limited by DPW using the parameters $k_a$, $\alpha_a$, $k_o$, and $\alpha_o$.
In the case where the simulated observation is rejected (line~\ref{lin:notnew}), the tree search is continued with an observation selected in proportion to the number of times, $M$, it has been previously simulated (line~\ref{lin:selecto}) and a state is sampled from the associated belief (line~\ref{lin:samples}).

\setcounter{algorithm}{0}
\begin{algorithm}
    \caption{POMCP-DPW} \label{alg:pomcpdpw}
    \begin{algorithmic}[1]
        \Procedure {Simulate}{$s$, $h$, $d$}        
            \If{$d = 0$}
                \State \textbf{return} $0$
            \EndIf
            \State $a \gets \Call{ActionProgWiden}{h}$
            \If{$|C(ha)| \leq k_o N(ha)^{\alpha_o}$}
                \State $s',o,r \gets G(s,a)$
                \State $C(ha) \gets C(ha) \cup \{o\}$
                \State $M(hao) \gets M(hao) + 1$
                \State $\text{append } s' \text{ to } B(hao)$ \label{lin:insertion}
                \If{$M(hao) = 1$}
                    \State $total \gets r + \gamma \Call{Rollout}{s', hao, d-1}$
                \Else
                    \State $total \gets r + \gamma \Call{Simulate}{s', hao, d-1}$
                \EndIf
            \Else \label{lin:notnew}
                \State $o \gets \text{select } o \in C(ha) \text{ w.p. } \frac{M(hao)}{\sum_{o} M(hao)}$ \label{lin:selecto}
                \State $s' \gets \text{select } s' \in B(hao) \text{ w.p. } \frac{1}{|B(hao)|}$ \label{lin:samples}
                \State $r \gets R(s,a,s')$
                \State $total \gets r + \gamma \Call{Simulate}{s', hao, d-1}$
            \EndIf
            \State $N(h) \gets N(h)+1$
            \State $N(ha) \gets N(ha)+1$
            \State $Q(ha) \gets Q(ha) + \frac{total - Q(ha)}{N(ha)}$
            \State \textbf{return} $total$
        \EndProcedure
    \end{algorithmic}
\end{algorithm}

This algorithm obtained remarkably good solutions for a very large autonomous freeway driving POMDP with multiple vehicles (up to 40 continuous fully observable state dimensions and 72 continuous correlated partially observable state dimensions) \cite{sunberg2017value}.
To our knowledge, that is the first work applying progressive widening to POMCP, and it does not contain a detailed description of the algorithm or any theoretical or experimental analysis other than the driving application.

This algorithm may converge to the optimal solution for POMDPs with discrete observation spaces; however, on continuous observation spaces, POMCP-DPW is suboptimal.
In particular, it finds a QMDP policy, that is, the solution under the assumption that the problem becomes fully observable after one time step \cite{littman1995learning,kochenderfer2015decision}.
In fact, for a modified version of POMCP-DPW, it is easy to prove analytically that it will converge to such a policy.
This is expressed formally in Theorem~\ref{thm:qmdp} below.
A complete description of the modified algorithm and problem requirements including the definitions of polynomial exploration, the regularity hypothesis for the problem, and exponentially sure convergence are given in \cref{sec:proof}.

\begin{definition}[QMDP value]
    Let $Q_\text{MDP}(s,a)$ be the optimal state-action value function assuming full observability starting by taking action $a$ in state $s$. The \emph{QMDP value} at belief $b$, $Q_\text{MDP}(b,a)$, is the expected value of $Q_\text{MDP}(s,a)$ when $s$ is distributed according to $b$.
\end{definition}

\begin{restatable}[Modified POMCP-DPW convergence to QMDP]{theorem}{qmdp}
    \label{thm:qmdp}
If a bounded-horizon POMDP meets the following conditions: 1) the state and observation spaces are continuous with a finite observation probability density function, and 2) the regularity hypothesis is met, then modified POMCP-DPW will produce a value function estimate, $\hat{Q}$, that converges to the QMDP value for the problem.
Specifically, there exists a constant $C>0$, such that after $n$ iterations,
\begin{equation*}
    \left| \hat{Q}(b,a) - Q_\text{MDP}(b,a) \right| \leq \frac{C}{n^{1/(10d_{\max}-7)}}
\end{equation*}
exponentially surely in $n$, for every action $a$.
\end{restatable}

A proof of this theorem that leverages work by \citet{auger2013continuous} is given in \cref{sec:proof}, but we provide a brief justification here.
The key is that belief nodes will contain only a single state particle (see \cref{fig:treecomp}).
This is because, since the observation space is continuous with a finite density function, the generative model will (with probability one) produce a unique observation $o$ each time it is queried.
Thus, for every generated history $h$, only one state will ever be inserted into $B(h)$ (line~\ref{lin:insertion}, \cref{alg:pomcpdpw}), and therefore $h$ is merely an alias for that state. 
Since each belief node corresponds to a state, the solver is actually solving the fully observable MDP at every node except the root node, leading to a QMDP solution.

As a result of \cref{thm:qmdp}, the action chosen by modified POMCP-DPW will match a QMDP policy (a policy of actions that maximize the QMDP value) with high precision exponentially surely (see Corollary 1 of \citet{auger2013continuous}).
For many problems this is a very useful solution,\footnote{Indeed, a useful online QMDP tree search algorithm could be created by deliberately constructing a tree with a single root belief node and fully observable state nodes below it.} but since it neglects the value of information, a QMDP policy is suboptimal for problems where information gathering is important \cite{littman1995learning,kochenderfer2015decision}.

Although \cref{thm:qmdp} is only theoretically applicable to the modified version of POMCP-DPW, it helps explain the behavior of other solvers.
Modified POMCP-DPW, POMCP-DPW, DESPOT, and ABT all share the characteristic that a belief node can only contain two states if they generated exactly the same observation.
Since this is an event with zero probability for a continuous observation space, these solvers exhibit suboptimal, often QMDP-like, behavior.
The experiments in \cref{sec:experiments} show this for POMCP-DPW and DESPOT, and this is presumably the case for ABT as well.

\begin{figure}[htpb]
    \centering
    \begin{subfigure}[b]{0.45\columnwidth}
        \centering
        \includegraphics[width=\textwidth]{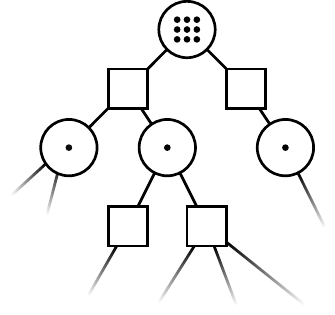}
        \caption{POMCP-DPW Tree}
    \end{subfigure}
    \begin{subfigure}[b]{0.45\columnwidth}
        \centering
        \includegraphics[width=\textwidth]{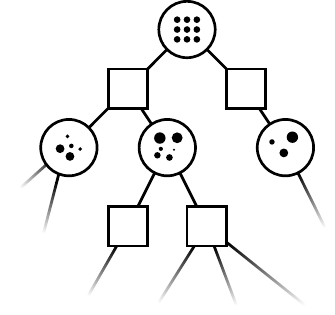}
        \caption{POMCPOW Tree}
    \end{subfigure}
    \caption{Tree Structure Comparison. Each square is an action node, and each unfilled circle is an observation node. Each black dot corresponds to a state particle with the size representing its weight. In continuous observation spaces, the beliefs in a POMCP-DPW tree degenerate to a single particle, while POMCPOW maintains weighted particle mixture beliefs.}
    \label{fig:treecomp}
\end{figure}

\subsection{POMCPOW}

\begin{algorithm}[t]
    \caption{POMCPOW} \label{alg:pomcpow}
    \begin{algorithmic}[1]
        \Procedure {Simulate}{$s$, $h$, $d$}        
            \If{$d = 0$}
                \State \textbf{return} $0$
            \EndIf
            \State $a \gets \Call{ActionProgWiden}{h}$
            \State $s',o,r \gets G(s,a)$
            \If{$|C(ha)| \leq k_o N(ha)^{\alpha_o}$}
                \State $M(hao) \gets M(hao) + 1$
            \Else
                \State $o \gets \text{select } o \in C(ha) \text{ w.p. } \frac{M(hao)}{\sum_{o} M(hao)}$
            \EndIf
            \State $\text{append } s' \text{ to } B(hao)$ \label{lin:insert}
            \State $\text{append } \obsdist(o \mid s, a, s') \text{ to } W(hao)$ \label{lin:weight}
            \If{$o \notin C(ha)$} \Comment{new node}
                \State $C(ha) \gets C(ha) \cup \{o\}$
                \State $total \gets r + \gamma \Call{Rollout}{s', hao, d-1}$
            \Else
                \State $s' \gets \text{select } B(hao)[i] \text{ w.p. } \frac{W(hao)[i]}{\sum_{j=1}^m W(hao)[j]}$ \label{lin:sample}
                \State $r \gets R(s,a,s')$
                \State $total \gets r + \gamma \Call{Simulate}{s', hao, d-1}$
            \EndIf
            \State $N(h) \gets N(h)+1$
            \State $N(ha) \gets N(ha)+1$
            \State $Q(ha) \gets Q(ha) + \frac{total - Q(ha)}{N(ha)}$
            \State \textbf{return} $total$
        \EndProcedure
    \end{algorithmic}
\end{algorithm}

In order to address the suboptimality of POMCP-DPW, we now propose a new algorithm, POMCPOW, shown in \cref{alg:pomcpow}.
In this algorithm, the belief updates are weighted, but they also expand gradually as more simulations are added.
Furthermore, since the richness of the belief representation is related to the number of times the node is visited, beliefs that are more likely to be reached by the optimal policy have more particles.
At each step, the simulated state is inserted into the weighted particle collection that represents the belief (line~\ref{lin:insert}), and a new state is sampled from that belief (line~\ref{lin:sample}).
A simple illustration of the tree is shown in \Cref{fig:treecomp} to contrast with a POMCP-DPW tree.
Because the resampling in line~\ref{lin:sample} can be efficiently implemented with binary search, the computational complexity is $\mathcal{O}(n d \log(n))$.

\subsection{PFT-DPW}

Another algorithm that one might consider for solving continuous POMDPs online is MCTS-DPW on the equivalent belief MDP.
Since the Bayesian belief update is usually computationally intractable, a particle filter is used.
This new approach will be referred to as particle filter trees with double progressive widening (PFT-DPW).
It is shown in \cref{alg:pft}, where $G_\text{PF($m$)}(b,a)$ is a particle filter belief update performed with a simulated observation and $m$ state particles which approximates the belief MDP generative model.
The authors are not aware of any mention of this algorithm in prior literature, but it is very likely that MCTS with particle filters has been used before without double progressive widening under another name.

PFT-DPW is fundamentally different from POMCP and POMCPOW because it relies on simulating approximate belief trajectories instead of state trajectories.
This distinction also allows it to be applied to problems where the reward is a function of the belief rather than the state such as pure information-gathering problems \cite{dressel2017efficient,araya2010pomdp}.

The primary shortcoming of this algorithm is that the number of particles in the filter, $m$, must be chosen a-priori and is static throughout the tree.
Each time a new belief node is created, an $\mathcal{O}(m)$ particle filter update is performed.
If $m$ is too small, the beliefs may miss important states, but if $m$ is too large, constructing the tree is expensive.
Fortunately, the experiments in \cref{sec:experiments} show that it is often easy to choose $m$ in practice; for all the problems we studied, a value of $m=20$ resulted in good performance.

\begin{algorithm}
    \caption{PFT-DPW} \label{alg:pft}
    \begin{algorithmic}[1]
        \Procedure{Plan}{$b$}
            \For{$i \in 1:n$}
                \State $\Call{Simulate}{b, d_\text{max}}$
            \EndFor
            \State $\textbf{return } \underset{a}{\argmax}\, Q(ba)$
        \EndProcedure
        \Procedure {Simulate}{$b$, $d$}        
            \If{$d = 0$}
                \State \textbf{return} $0$
            \EndIf
            \State $a \gets \Call{ActionProgWiden}{b}$
            \If{$|C(ba)| \leq k_o N(ba)^{\alpha_o}$}
                \State $b',r \gets G_\text{PF($m$)}(b,a)$
                \State $C(ba) \gets C(ba) \cup \{(b',r)\}$
                \State $total \gets r + \gamma \Call{Rollout}{b', d-1}$
            \Else
                \State $b', r \gets \text{sample uniformly from } C(ba)$
                \State $total \gets r + \gamma \Call{Simulate}{b', d-1}$
            \EndIf
            \State $N(b) \gets N(b)+1$
            \State $N(ba) \gets N(ba)+1$
            \State $Q(ba) \gets Q(ba) + \frac{total - Q(ba)}{N(ba)}$
            \State \textbf{return} $total$
        \EndProcedure
    \end{algorithmic}
\end{algorithm}

\subsection{Observation Distribution Requirement}

It is important to note that, while POMCP, POMCP-DPW, and DESPOT only require a generative model of the problem,  both POMCPOW and PFT-DPW require a way to query the relative likelihood of different observations ($\obsdist$ in line~\ref{lin:weight}).
One may object that this will limit the application of POMCPOW to a small class of POMDPs, but we think it will be an effective tool in practice for two reasons.

First, this requirement is no more stringent than the requirement for a standard importance resampling particle filter, and such filters are used widely, at least in the field of robotics that the authors are most familiar with. 
Moreover, if the observation model is complex, an approximate model may be sufficient.

Second, given the implications of \cref{thm:qmdp}, it is difficult to imagine a tree-based decision-making algorithm or a robust belief updater that does not require some way of measuring whether a state belongs to a belief or history.
The observation model is a straightforward and standard way of specifying such a measure.
Finally, in practice, except for the simplest of problems, using POMCP or DESPOT to repeatedly observe and act in an environment already requires more than just a generative model.
For example, the authors of the original paper describing POMCP~\cite{silver2010pomcp} use heuristic particle reinvigoration in lieu of an observation model and importance sampling.

\section{Experiments} \label{sec:experiments}

Numerical simulation experiments were conducted to evaluate the performance of POMCPOW and PFT-DPW compared to other solvers. The open source code for the experiments is built on the POMDPs.jl framework \cite{egorov2017pomdps} and is hosted at \url{https://github.com/zsunberg/ContinuousPOMDPTreeSearchExperiments.jl}. In all experiments, the solvers were limited to \SI{1}{second} of computation time per step. Belief updates were accomplished with a particle filter independent of the planner, and no part of the tree was saved for re-use on subsequent steps. Hyperparameter values are shown in \cref{sec:hyper}.

\begin{table*}
    {\centering
        \caption{Experimental Results} \label{tab:experiments}

\begin{tabularx}{\linewidth}{lXrlXrlXrl}
\toprule
& & Laser Tag \makebox[0pt][l]{(D, D, D)} & & & Light Dark \makebox[0pt][l]{(D, D, C)} & & & Sub Hunt \makebox[0pt][l]{(D, D, C)} & \\
\midrule
POMCPOW & & \result{-10.3}{0.2}{81}{0.81} & & \result{56.1}{0.6}{76}{0.76} & & \result{69.2}{1.3}{87}{0.87} \\
PFT-DPW & & \result{-11.1}{0.2}{74}{0.74} & & \result{57.2}{0.5}{77}{0.77} & & \result{77.4}{1.1}{97}{0.97} \\
QMDP & & \result{-10.5}{0.2}{79}{0.79} & & \result{-6.4}{1.0}{14}{0.14} & & \result{28.0}{1.3}{35}{0.35} \\
POMCP-DPW & & \result{-10.6}{0.2}{78}{0.78} & & \result{-7.3}{1.0}{13}{0.13} & & \result{28.3}{1.3}{35}{0.35} \\
DESPOT & & \result{-8.9}{0.2}{92}{0.92} & & \result{-6.8}{1.0}{13}{0.13} & & \result{26.8}{1.3}{34}{0.34} \\
POMCP\textsuperscript{D} & & \result{-14.1}{0.2}{49}{0.49} & & \result{61.1}{0.4}{81}{0.81} & & \result{28.0}{1.3}{35}{0.35} \\
DESPOT\textsuperscript{D} & & \noresult{} & & \result{54.2}{1.1}{74}{0.74} & & \result{27.4}{1.3}{34}{0.34} \\

\midrule
& & VDP Tag \makebox[0pt][l]{(C, C, C)} & & & Multilane \makebox[0pt][l]{(C, D, C)} & \\
\midrule
POMCPOW & & \result{29.3}{0.8}{95}{0.95} & & \result{30.9}{0.9}{70}{0.70} \\
PFT-DPW & & \result{27.2}{0.8}{88}{0.88} & & \result{21.4}{0.9}{38}{0.38} \\
QMDP & & \noresult{} & & \noresult{} \\
POMCP-DPW & & \result{16.4}{1.0}{53}{0.53} & & \result{29.6}{0.9}{65}{0.65} \\
DESPOT & & \noresult{} & & \result{36.0}{0.8}{87}{0.87} \\
POMCP\textsuperscript{D} & & \result{14.7}{0.9}{47}{0.47} & & \noresult{} \\
DESPOT\textsuperscript{D} & & \result{14.3}{1.0}{46}{0.46} & & \noresult{} \\

\bottomrule
\end{tabularx}

    }
    \vspace{1mm}
    \footnotesize{The three C or D characters after the solver indicate whether the state, action, and observation spaces are continuous or discrete, respectively. For continuous problems, solvers with a superscript D were run on a version of the problem with discretized action and observation spaces, but they interacted with continuous simulations of the problem.}
\end{table*}

\subsection{Laser Tag}

The Laser Tag benchmark is taken directly from the work of \citet{somani2013despot} and included for the sake of calibration. DESPOT outperforms the other methods. The score for DESPOT differs slightly from that reported by \citet{somani2013despot} likely because of bounds implementation differences.
POMCP performs much better than reported by \citet{somani2013despot} because this implementation uses a state-based rollout policy.

\subsection{Light Dark}

In the Light Dark domain, the state is an integer, and the agent can choose how to move deterministically ($s' = s+a$) from the action space $\mathcal{A}=\{-10, -1, 0, 1, -10\}$. 
The goal is to reach the origin.
If action $0$ is taken at the origin, a reward of $100$ is given and the problem terminates; If action $0$ is taken at another location, a penalty of $-100$ is given.
There is a cost of $-1$ at each step before termination.
The agent receives a more accurate observation in the ``light'' region around $s=10$.
Specifically, observations are continuous ($\obsspace=\reals$) and normally distributed with standard deviation $\sigma=|s-10|$.

\Cref{tab:experiments} shows the mean reward from $1000$ simulations for each solver, and \cref{fig:ld} shows an example experiment.
The optimal strategy involves moving toward the light region and localizing before proceeding to the origin. 
QMDP and solvers predicted to behave like QMDP attempt to move directly to the origin, while POMCPOW and PFT-DPW perform better.
In this one-dimensional case, discretization allows POMCP to outperform all other methods and DESPOT to perform well, but in subsequent problems where the observation space has more dimensions, discretization does not provide the same performance improvement (see \cref{sec:discretization}).

\begin{figure}[htb]
    \centering
    \includegraphics[width=\columnwidth]{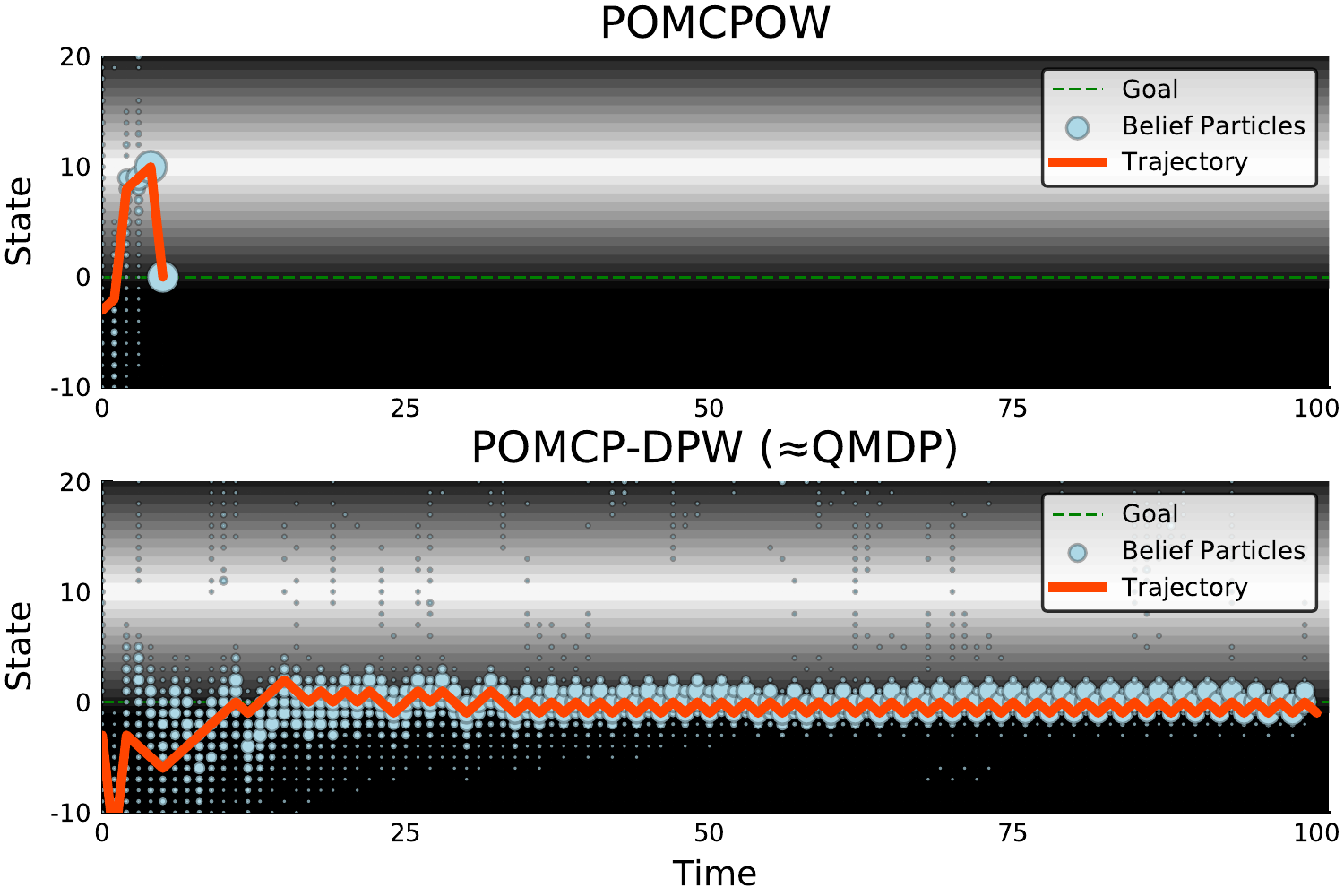}
    \caption{Example trajectories in the Light Dark domain. POMCPOW travels to the light region and accurately localizes before moving to the goal. POMCP-DPW displays QMDP-like behavior: it is unable to localize well enough to take action \num{0} with confidence. The belief particles far away from \num{0} in the POMCP-DPW plot are due to particle reinvigoration that makes the filter more robust.}
    \label{fig:ld}
\end{figure}

\subsection{Sub Hunt}

In the Sub Hunt domain, the agent is a submarine attempting to track and destroy an enemy sub.
The state and action spaces are discrete so that QMDP can be used to solve the problem for comparison.
The agent and the target each occupy a cell of a 20 by 20 grid. The target is either aware or unaware of the agent and seeks to reach a particular edge of the grid unknown to the agent ($\mathcal{S} = \{1,..,20\}^4 \times \{\text{aware}, \text{unaware}\} \times \{N,S,E,W\}$).
The target stochastically moves either two steps towards the goal or one step forward and one to the side.
The agent has six actions, move three steps north, south, east, or west, engage the other submarine, or ping with active sonar.
If the agent chooses to engage and the target is unaware and within a range of 2, a hit with reward 100 is scored; The problem ends when a hit is scored or the target reaches its goal edge.

An observation consists of 8 sonar returns ($\obsspace = \reals^8$) at equally-spaced angles that give a normally distributed estimate ($\sigma=0.5$) of the range to the target if the target is within that beam and a measurement with higher variance if it is not.
The range of the sensors depends on whether the agent decides to use active sonar.
If the agent does not use active sonar it can only detect the other submarine within a radius of 3, but pinging with active sonar will detect at any range.
However, active sonar alerts the target to the presence of the agent, and when the target is aware, the hit probability when engaging drops to $60\%$.

\Cref{tab:experiments} shows the mean reward for $1000$ simulations for each solver.
The optimal strategy includes using the active sonar, but previous approaches have difficulty determining this because of the reduced engagement success rate.
The PFT-DPW approach has the best score, followed closely by POMCPOW.
All other solvers have similar performance to QMDP.

\subsection{Van Der Pol Tag} \label{sec:vdptag}

\begin{figure}[bth]
    \centering
    \begin{minipage}{0.5\linewidth}
        \centering
        \includegraphics[width=0.8\linewidth]{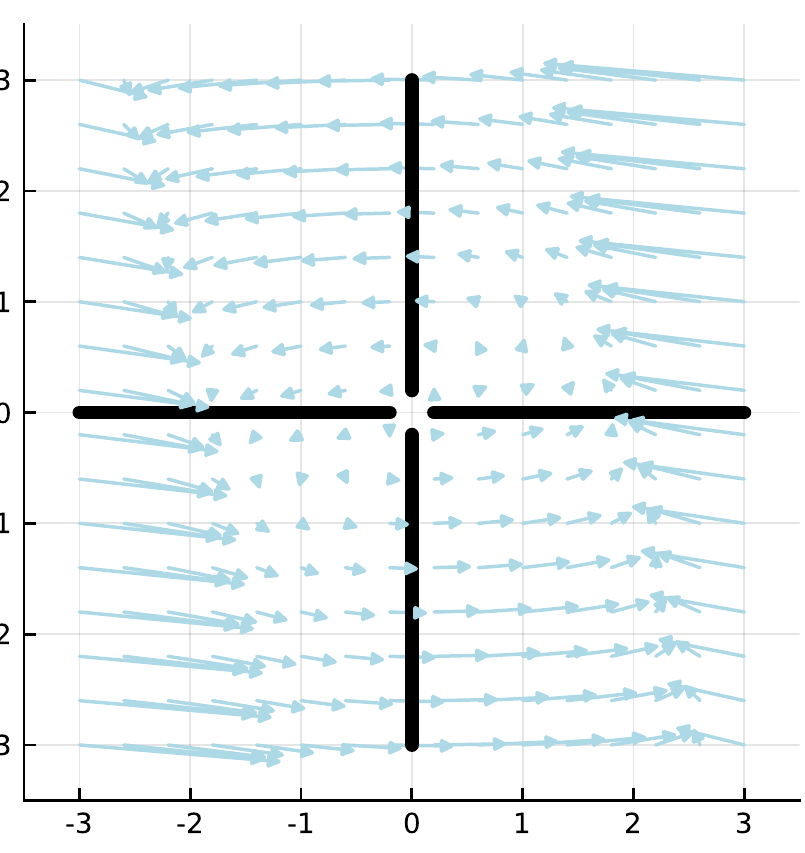}
    \end{minipage}
    \hfill
    \begin{minipage}[c]{0.45\linewidth}
        \caption{Van Der Pol tag problem. The arrows show the target differential equation, and the thick black lines represent the barriers.}
        \label{fig:vdp}
    \end{minipage}
\end{figure}

The only experimental problem with continuous state, action, and observation spaces is called Van Der Pol tag.
In this problem an agent moves through 2D space to try to tag a target ($\mathcal{S}=\reals^4$) that has a random unknown initial position in $[-4, 4]\times[-4,4]$.
The agent always travels at the same speed, but chooses a direction of travel and whether to take an accurate observation ($\mathcal{A} = [0, 2\pi)\times\{0,1\}$).
The observation again consists of 8 beams ($\obsspace=\reals^8$) that give measurements to the target.
Normally, these measurements are too noisy to be useful ($\sigma=5$), but, if the agent chooses an accurate measurement with a cost of \num{5}, the observation has low noise ($\sigma=0.1$).
The agent is blocked if it comes into contact with one of the barriers that stretch from \num{0.2} to \num{3.0} in each of the cardinal directions (see \cref{fig:vdp}), while the target can move freely through.
There is a cost of \num{1} for each step, and a reward of \num{100} for tagging the target (being within a distance of \num{0.1}).

The target moves following a two dimensional form of the Van Der Pol oscillation defined by the differential equations%
\begin{equation}
    \dot{x} = \mu \left( x - \frac{x^3}{3} -y \right) \quad \text{ and }\quad \dot{y} = \frac{1}{\mu}x\text{,} \nonumber
\end{equation}
where $\mu=2$.
Gaussian noise ($\sigma=0.05$) is added to the position at the end of each step.
Runge-Kutta fourth order integration is used to propagate the state.

This problem has several challenging features that might be faced in real-world applications.
First, the state transitions are more computationally expensive because of the numerical integration.
Second, the continuous state space and obstacles make it difficult to construct a good heuristic rollout policy, so random rollouts are used.
\Cref{tab:experiments} shows the mean reward for $1000$ simulations of this problem for each solver.
Since a POMCPOW iteration requires less computation than a PFT-DPW iteration, POMCPOW simulates more random rollouts and thus performs slightly better.

\subsection{Multilane} \label{sec:multilane}

\begin{figure*}[htbp]
    \centering
    \includegraphics[width=\linewidth]{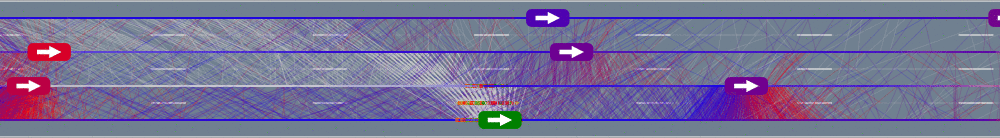}
    \caption{Visualization of POMCPOW planning lane changes in the ``multilane'' example. The green vehicle is controlled by POMCPOW, and the white lines represent its potential trajectories in the POMCPOW tree. The color of the other vehicles represents the aggressiveness of their human driver models, and the red and blue lines proceeding out from them indicate their predicted trajectories.}
    \label{fig:multilane}
\end{figure*}

In the ``multilane'' environment, an autonomous vehicle must make a series of lane changes without causing any unsafe situations.
The other vehicles are governed by a noisy version of the intelligent driver model (IDM)~\cite{treiber2000idm} for longitudinal motion and the MOBIL model~\cite{kesting2007mobil} for lane changing.
The physical states of all vehicles are fully observable, and the latent states of the POMDP are the parameters for the IDM and MOBIL models.
The autonomous vehicle must make three lane changes within \SI{1000}{m}, and avoid any dangerous situations, defined as speed less than \SI{15}{\meter/\second} or braking harder than \SI{4}{\meter/\second\squared}.
This is similar to the problem presented by~\citet{sunberg2017value}, and more details about these specific results can be found in the thesis by \citet{sunberg2018thesis}.
A visualization of POMCPOW planning in this domain is shown in \cref{fig:multilane}.

In this environment, POMCPOW significantly outperforms PFT-DPW.
Since there are no costly information-gathering actions, the QMDP solution is adequate.
Thus, the depth and quality of the search tree is more important than maintaining rich belief representations, and DESPOT with its bound-based exploration and limited number of scenarios is able to outperform all other approaches.

\section{Conclusion}

In this paper, we have proposed a new general-purpose online POMDP algorithm that is able to solve problems with continuous state, action, and observation spaces.
This is a qualitative advance in capability over previous solution techniques, with the only major new requirement being explicit knowledge of the observation distribution.

This study has yielded several insights into the behavior of tree search algorithms for POMDPs.
We explained why POMCP-DPW and other solvers are unable to choose costly information-gathering actions in continuous spaces, and showed that POMCPOW and PFT-DPW are both able to overcome this challenge.
In the most realistic multilane driving test (\cref{sec:multilane}), POMCPOW significantly outperformed PFT-DPW but was itself outperformed by DESPOT because sufficient information could be gathered passively.

The theoretical properties of the algorithms remain to be proven.
In addition, better ways for choosing continuous actions would provide an improvement.
The techniques that others have studied for handling continuous actions such as generalized pattern search \cite{seiler2015online} and hierarchical optimistic optimization \cite{mansley2011sample} are complimentary to this work, and the combination of these approaches will likely yield powerful tools for solving real problems.

\subsubsection*{Acknowledgements}

Toyota Research Institute (``TRI'')  provided funds to assist the authors with their research, but this article solely reflects the opinions and conclusions of its authors and not TRI or any other Toyota entity.

The authors would also like to thank Zongzhang Zhang for his especially helpful comments and Auke Wiggers for catching several pseudocode mistakes.



\bibliographystyle{aaai}
\bibliography{ref}


\vfill

\begin{appendices}

\crefalias{section}{appsec}

\section{Discretization} \label{sec:discretization}

Discretization is perhaps the most straightforward way to deal with continuous observation spaces. The results in \cref{tab:experiments} show that this approach is only sometimes effective.
\Cref{fig:disc} shows the performance at different discretization granularities for the Light Dark and Sub Hunt problems.

Since the Light Dark domain has only a single observation dimension, it is easy to discretize.
In fact, POMCP with fine discretization outperforms POMCPOW.
However, discretization is only effective at certain granularities, and this is highly dependent on the solver and possibly hyperparameters.
In the Sub Hunt problem, with its high-dimensional observation, discretization is not effective at any granularity.
In Van Der Pol tag, both the action and observation spaces must be discretized.
Due to the high dimensionality of the observation space, similar to Sub Hunt, no discretization that resulted in good performance was found.

\begin{figure}[htb]
    \centering
    \begin{subfigure}{\linewidth}
    \centering
    \includegraphics[width=0.8\linewidth]{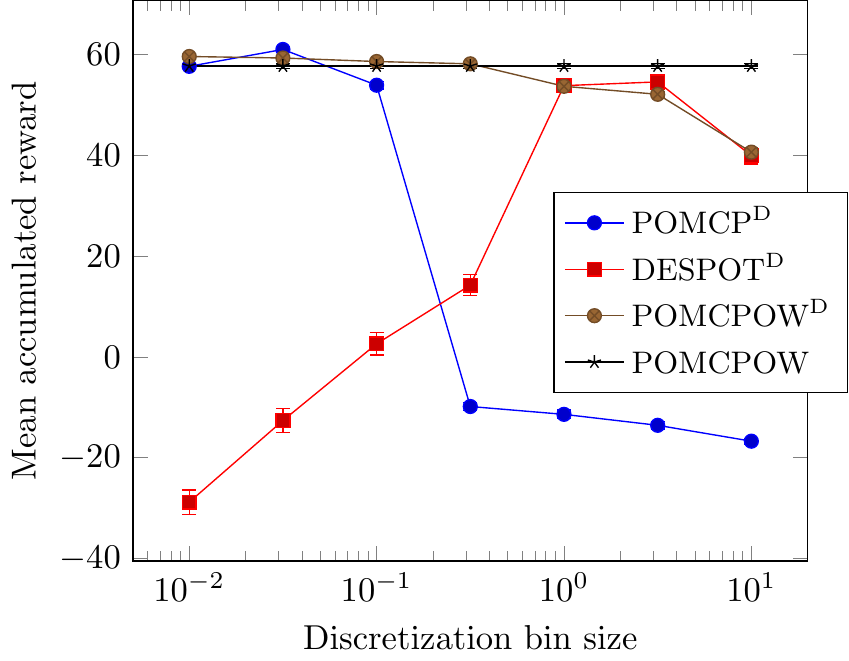}
    \caption{Light Dark} \label{fig:lddisc}

    \end{subfigure}

    \begin{subfigure}{\linewidth}
    \centering
    \includegraphics[width=0.72\linewidth]{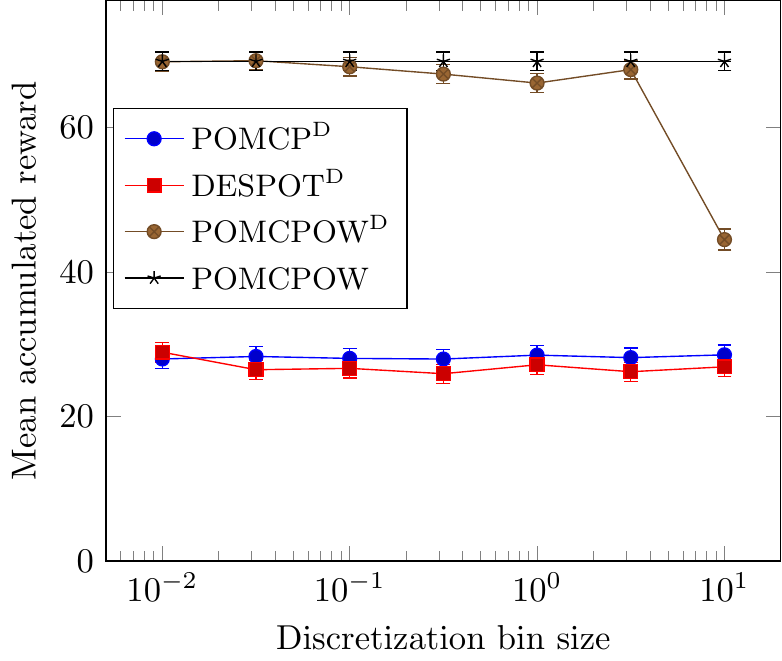}
    \caption{Sub Hunt}
    \label{fig:shdisc}
    \end{subfigure}

    \caption{Discretization granularity studies} \label{fig:disc}
\end{figure}

\section{Hyperparameters} \label{sec:hyper}

\begin{table}[t]
    {\centering
\caption{Hyperparameters used in experiments} \label{tab:hyper}

\begin{tabular}{lrrrr}
    \toprule
                & Laser Tag     & Light Dark    & Sub Hunt      & VDP Tag \\
    \midrule
    \multicolumn{3}{l}{POMCPOW} \\
    \midrule
    $c$         & \num{26.0}    & \num{90.0}    & \num{17.0}    & \num{110.0} \\
    $k_a$       & --            & --            & --            & \num{30.0}  \\
    $\alpha_a$  & --            & --            & --            & \num{1/30}  \\
    $k_o$       & \num{4.0}     & \num{5.0}     & \num{6.0}     & \num{5.0}   \\
    $\alpha_o$  & \num{1/35}    & \num{1/15}    & \num{1/100}   & \num{1/100} \\
    \midrule
    \multicolumn{3}{l}{PFT-DPW} \\
    \midrule
    $m$         & \num{20}      & \num{20}      & \num{20}      & \num{20}    \\
    $c$         & \num{26.0}    & \num{100.0}   & \num{100.0}   & \num{70.0}  \\
    $k_a$       & --            & --            & --            & \num{20.0}  \\
    $\alpha_a$  & --            & --            & --            & \num{1/25}  \\
    $k_o$       & \num{4.0}     & \num{4.0}     & \num{2.0}     & \num{8.0}   \\
    $\alpha_o$  & \num{1/35}    & \num{1/10}    & \num{1/10}    & \num{1/85}  \\
    \midrule

\end{tabular}
    }
    \vspace{1mm}

    \footnotesize{For problems with discrete actions, all actions are considered and $k_a$ and $\alpha_a$ are not needed.}
\end{table}

Hyperparameters for POMCPOW and PFT-DPW were chosen using the cross entropy method \cite{mannor2003cross}, but exact tuning was not a high priority and some parameters were re-used across solvers so the parameters may not be perfectly optimized.
The values used in the experiments are shown in \cref{tab:hyper}. 
There are not enough experiments to draw broad conclusions about the hyperparameters, but it appears that performance is most sensitive to the exploration constant, $c$.

The values for the observation widening parameters, $k_o$ and $\alpha_o$, were similar for all the problems in this work.
A small $\alpha_o$ essentially limits the number of observations to a static number $k_o$, resulting in behavior reminiscent of sparse UCT \cite{browne2012survey}, preventing unnecessary widening and allowing the tree to grow deep.
This seems to work well in practice with the branching factor ($k_o$) set to values between \num{2} and \num{8}, and suggests that it may be sufficient to limit the number of children to a fixed number rather than do progressive widening in a real implementation.

\section{Proof of Theorem 1} \label{sec:proof}

A version of Monte Carlo tree search with double progressive widening has been proven to converge to the optimal value function on fully observable MDPs by \citet{auger2013continuous}.
We utilize this proof to show that POMCP-DPW converges to a solution that is sometimes suboptimal.

First we establish some preliminary definitions taken directly from \citet{auger2013continuous}.

\begin{definition}[Regularity Hypothesis]
    The \emph{Regularity hypothesis} is the assumption that for any $\Delta > 0$, there is a non zero probability to sample an action that is optimal with precision $\Delta$. More precisely, there is a $\theta > 0$ and a $p > 1$ (which remain the same during the whole simulation) such that for all $\Delta > 0$, 
\begin{align}
    Q(ha) & \geq Q^*(h)-\Delta \nonumber{}\\ & \text{ with probability at least } \min(1, \theta \Delta^p)\text{.}
\end{align}
\end{definition}

\begin{definition}[Exponentially sure in $n$]
    We say that some property depending on an integer $n$ is exponentially sure in $n$ if there exists positive constants $C$, $h$, and $\eta$ such that the probability that the property holds is at least $$1-C \exp(-hn^\eta)\text{.}$$
\end{definition}

In order for the proof from \citet{auger2013continuous} to apply, the following four minor modifications to the POMCP-DPW algorithm must be made: 

\begin{enumerate}
    \item Instead of the usual logarithmic exploration, use \emph{polynomial exploration}, that is, select actions based on the criterion
    \begin{equation}
        Q(ha) + \sqrt{\frac{N(h)^{e_d}}{N(ha)}}
    \end{equation}
    as opposed to the traditional criterion
    \begin{equation}
        Q(ha) + c \sqrt{\frac{\log N(h)}{N(ha)}}\text{,}
    \end{equation}
    and create a new node for progressive widening when $\lfloor N^\alpha \rfloor > \lfloor (N-1)^\alpha \rfloor$ rather than when the number of children exceeds $k N^\alpha$.

    \item Instead of performing rollout simulations, keep creating new single-child nodes until the maximum depth is reached.

    \item In line~\ref{lin:selecto}, instead of selecting an observation randomly, select the observation that has been visited least proportionally to how many times it's parent has been visited.

    \item Use the depth-dependent coefficient values in Table 1 from \citet{auger2013continuous} instead of choosing static values.
\end{enumerate}

This version of the algorithm will be referred to as ``modified POMCP-DPW''. The algorithm with these changes is listed in \Cref{alg:mpomcpdpw}.

\begin{algorithm}[htb]
    \caption{Modified POMCP-DPW} \label{alg:mpomcpdpw}
    \begin{algorithmic}[1]
        \Procedure{Plan}{$b$}
            \For{$i \in 1:n$}
                \State $s \gets \text{sample from }b$ \label{lin:msample}
                \State $\Call{Simulate}{s, b, d_\text{max}}$
            \EndFor
            \State $\textbf{return } \underset{a}{\argmax}\, Q(ba)$
        \EndProcedure

        \Procedure {ActionProgWiden}{$h$}
            \If{$\lfloor N(h)^{\alpha_{a,d}} \rfloor > \lfloor (N(h)-1)^{\alpha_{a,d}} \rfloor$}
                \State $a \gets \Call{NextAction}{h}$
                \State $C(h) \gets C(h) \cup \{a\}$
            \EndIf
            \State $\textbf{return } \underset{a \in C(h)}{\argmax}\, Q(ha) + \sqrt{\frac{N(h)^{e_d}}{N(ha)}}$
        \EndProcedure

        \Procedure {Simulate}{$s$, $h$, $d$}        
            \If{$d = 0$}
                \State \textbf{return} $0$
            \EndIf
            \State $a \gets \Call{ActionProgWiden}{h}$
            \If{$\lfloor N(ha)^{\alpha_{o,d}} \rfloor > \lfloor (N(ha)-1)^{\alpha_{o,d}} \rfloor$}
                \State $s',o,r \gets G(s,a)$
                \State $C(ha) \gets C(ha) \cup \{o\}$
                \State $M(hao) \gets M(hao) + 1$ \label{lin:gencount}
                \State $\text{append } s' \text{ to } B(hao)$ \label{lin:minsertion}
            \Else
                \State $o \gets \underset{o \in C(ha)}{\argmin}\, N(hao)/M(hao)$
                \State $s' \gets \text{select } s' \in B(hao) \text{ w.p. } \frac{1}{|B(hao)|}$
                \State $r \gets R(s,a,s')$
            \EndIf
            \State $total \gets r + \gamma \Call{Simulate}{s', hao, d-1}$
            \State $N(h) \gets N(h)+1$
            \State $N(ha) \gets N(ha)+1$
            \State $Q(ha) \gets Q(ha) + \frac{total - Q(ha)}{N(ha)}$
            \State \textbf{return} $total$
        \EndProcedure
    \end{algorithmic}        
\end{algorithm}

We now define the ``QMDP value'' that POMCP-DPW converges to (this is repeated from the main text of the paper) and prove a preliminary lemma.

\begin{definition}[QMDP value]
     Let $Q_\text{MDP}(s,a)$ be the optimal state-action value function assuming full observability starting by taking action $a$ in state $s$.
     The \emph{QMDP value} at belief $b$, $Q_\text{MDP}(b,a)$, is the expected value of $Q_\text{MDP}(s,a)$ when $s$ is distributed according to $b$.   
\end{definition}

\begin{lemma} \label{lem:onestate}
    If POMCP-DPW or modified POMCP-DPW is applied to a POMDP with a continuous observation space and observation probability density functions that are finite everywhere, then each history node in the tree will have only one corresponding state, that is $|B(h)| = 1, M(h)=1\, \forall h$.
\end{lemma}

\begin{proof}
    Since the observation probability density function is finite, each call to the generative model will produce a unique observation with probability 1.
    Because of this, lines~\ref{lin:gencount}~and~\ref{lin:minsertion} of \cref{alg:mpomcpdpw} will only be executed once for each observation.
\end{proof}

We are now ready to restate and prove the theorem from the text.


\qmdp*

\begin{proof}
    We prove that modified POMCP-DPW functions exactly as the Polynomial UCT (PUCT) algorithm defined by \citet{auger2013continuous} applied to an augmented fully observable MDP, and hence converges to the QMDP value.
    We will show this by proposing incremental changes to \cref{alg:mpomcpdpw} that do not change its function that will result in an algorithm identical to PUCT.

    Before listing the changes, we define the ``augmented fully observable MDP" as follows: For a POMDP $\mathcal{P} = (\mathcal{S}, \mathcal{A}, \mathcal{T}, \mathcal{R}, \mathcal{O}, \mathcal{Z}, \gamma)$, and belief $b$, the \emph{augmented fully observable MDP}, $\mathcal{M}$, is the MDP defined by $(\mathcal{S}_A, \mathcal{A}, \mathcal{T}_A, \mathcal{R}, \gamma)$, where 
    \begin{equation}
        \mathcal{S}_A = \mathcal{S} \cup \{b\}
    \end{equation}
    and, for all $x, x' \in \mathcal{S}_A$,
    \begin{equation}
        \mathcal{T}_A (x'|x, a) = \begin{cases}
                \mathcal{T} (x' | x, a) & \text{if } x \in \mathcal{S} \\
                \int_S b(s) \mathcal{T} (x' | s, a) ds & \text{if } x = b
        \end{cases}
    \end{equation}
    This is simply the fully observable MDP augmented with a special state representing the current belief.
    It is clear that the value function for this problem $Q_\mathcal{M}(b, a)$ is the same as the QMDP value for the POMDP, $Q_\text{MDP}(b,a)$.
    Thus, by showing that modified POMCP-DPW behaves exactly as PUCT applied to $\mathcal{M}$, we show that it estimates the QMDP values.


    Consider the following modifications to \cref{alg:mpomcpdpw} that do not change its behavior when the observation space is continuous:

    \begin{enumerate}
        \item Eliminate the state count $M$. \emph{Justification}: By \cref{lem:onestate}, its value will be 1 for every node.
        \item Remove $B$ and replace with a mapping $H$ from each node to a state of $\mathcal{M}$; define $H(b) = b$. \emph{Justification}: By \cref{lem:onestate} $B$ always contains only a single state, so $H$ contains the same information.
        \item Generate states and rewards with $G_\mathcal{M}$, the generative model of $\mathcal{M}$, instead of $G$. \emph{Justification}: Since the state transition model for the fully observable MDP is the same as the POMDP, these are equivalent for all $s \in \mathcal{S}$.
        \item Remove the $s$ argument of \textproc{Simulate}. \emph{Justification}: The sampling in line~\ref{lin:msample} is done implicitly in $G_\mathcal{M}$ if $h=b$, and $s$ is redundant in other cases because $h$ can be mapped to $s$ through $H$.
    \end{enumerate}

    The result of these changes is shown in \cref{alg:c}. It is straightforward to verify that this algorithm is equivalent to PUCT applied to $\mathcal{M}$.
    Each observation-terminated history, $h$, corresponds to a PUCT ``decision node'', $z$, and each action-terminated history, $ha$, corresponds to a PUCT ``chance node'', $w$.
    In other words, the observations have no meaning in the tree other than making up the histories, which are effectively just keys or aliases for the state nodes.
    
    Since PUCT is guaranteed by Theorem 1 of \citet{auger2013continuous} to converge to the optimal value function of $\mathcal{M}$ exponentially surely, POMCP-DPW is guaranteed to converge to the QMDP value exponentially surely, and the theorem is proven.

\begin{algorithm}[htb]
    \caption{Modified POMCP-DPW on a continuous observation space} \label{alg:c}
    \begin{algorithmic}[1]
        \Procedure{Plan}{$b$}
            \For{$i \in 1:n$}
                \State $\Call{Simulate}{(b), d_\text{max}}$
            \EndFor
            \State $\textbf{return } \underset{a}{\argmax}\, Q(ha)$
        \EndProcedure

        \Procedure {ActionProgWiden}{$h$}
            \If{$\lfloor N(h)^{\alpha_{a,d}} \rfloor > \lfloor (N(h)-1)^{\alpha_{a,d}} \rfloor$}
                \State $a \gets \Call{NextAction}{h}$
                \State $C(h) \gets C(h) \cup \{a\}$
            \EndIf
            \State $\textbf{return } \underset{a \in C(h)}{\argmax}\, Q(ha) + \sqrt{\frac{N(h)^{e_d}}{N(ha)}}$
        \EndProcedure

        \Procedure {Simulate}{$h$, $d$}        
            \If{$d = 0$}
                \State \textbf{return} $0$
            \EndIf
            \State $a \gets \Call{ActionProgWiden}{h, d}$
            \If{$\lfloor N(ha)^{\alpha_{o,d}} \rfloor > \lfloor (N(ha)-1)^{\alpha_{o,d}} \rfloor$}
                \State $\cdot, o, \cdot \gets G(H(h), a)$
                \State $H(hao),r \gets G_\mathcal{M}(H(h),a)$
                \State $C(ha) \gets C(ha) \cup \{o\}$
            \Else
                \State $o \gets \underset{o \in C(ha)}{\argmin}\, N(hao)$
                \State $r \gets R(H(h), a, H(hao))$
            \EndIf
            \State $total \gets r + \gamma \Call{Simulate}{hao, d-1}$
            \State $N(h) \gets N(h)+1$
            \State $N(ha) \gets N(ha)+1$
            \State $Q(ha) \gets Q(ha) + \frac{total - Q(ha)}{N(ha)}$
            \State \textbf{return} $total$
        \EndProcedure
    \end{algorithmic}
\end{algorithm}

\end{proof}

\begin{remark}
    One may object that multiple histories may map to the same state through $H$, and thus the history nodes in a modified POMCP-DPW tree are not equivalent to state nodes in the PUCT tree. In fact, the PUCT algorithm does not check to see if a state has previously been generated by the model, so it may also contain multiple decision nodes $z$ that correspond to the same state. Though this is not explicitly stated by the authors, it is clear from the algorithm description, and the proof still holds.
\end{remark}

\end{appendices}

\end{document}